\documentclass[11pt]{article}
\usepackage[margin=1in]{geometry}

\usepackage[english]{babel}
\usepackage[utf8]{inputenc}
\usepackage{natbib}
\usepackage{amsmath}
\usepackage{amsthm, amssymb}
\usepackage{amsfonts}
\usepackage{bbm}
\usepackage{graphicx}
\usepackage{thmtools}
\usepackage{thm-restate}
\usepackage{color}
\usepackage[]{color-edits}
\addauthor{hb}{magenta}
\usepackage[linesnumbered,ruled,vlined,noend]{algorithm2e}
\usepackage{algorithmicx}
\usepackage{algpseudocode}

\algtext*{EndWhile}
\algtext*{EndFor}
\algtext*{EndIf}

\algnewcommand\algorithmicinput{\textbf{Input:}}
\algnewcommand\input{\item[\algorithmicinput]}
\algnewcommand\algorithmicoutput{\textbf{Output:}}
\algnewcommand\OUTPUT{\item[\algorithmicoutput]}
\usepackage{xfrac}
\usepackage{bbm}
\usepackage{enumitem}
\usepackage[normalem]{ulem}
\usepackage{hyperref}
\usepackage[nameinlink]{cleveref}
\usepackage{todonotes}
\usepackage{tikz}
\usetikzlibrary{intersections}
\usepackage{bm}
\usepackage{nicefrac}
\usepackage{url}

\usepackage{thmtools} 
\usepackage{thm-restate}

\newtheorem{theorem}{Theorem}

\newtheorem{definition}[theorem]{Definition}
\newtheorem{lemma}[theorem]{Lemma}
\newtheorem{corollary}[theorem]{Corollary}

\usepackage{sidecap, caption}
\usepackage{wrapfig}
\usepackage{tabularx}
\usepackage{booktabs}

\newcommand{\eps}{\varepsilon}
\renewcommand{\P}{\mathbb{P}}

\newcommand \R {\mathbb{R}}

\renewcommand \vec [1]{\bm{#1}}

\DeclareMathOperator{\sign}{sgn}

\renewcommand \vec [1]{\bm{#1}}

\DeclareMathAlphabet{\pazocal}{OMS}{zplm}{m}{n}

\renewcommand \vec [1]{\bm{#1}}
\renewcommand{\P}{\mathbb{P}}

\newcommand{\err}{\text{err}}

\title{The Active and Noise-Tolerant Strategic Perceptron}
\author{
	Maria-Florina Balcan \\ \small Carnegie Mellon University \\ \small \texttt{ninamf@cs.cmu.edu}
	\and 
	Hedyeh Beyhaghi \\ \small University of Massachusetts Amherst \\ \small \texttt{hbeyhaghi@umass.edu}}
\date{}
\begin{document}

\maketitle

\begin{abstract}

We initiate the study of {\em active learning algorithms for classifying strategic agents}. Active learning is a well-established framework in machine learning in which the learner selectively queries labels, often achieving substantially higher accuracy and efficiency than classical supervised methods—especially in settings where labeling is costly or time-consuming, such as hiring, admissions, and loan decisions. Strategic classification, on the other hand, addresses scenarios where agents modify their features to obtain more favorable outcomes, resulting in observed data that is not truthful. Such manipulation introduces challenges beyond those in learning from clean data. Our goal is to design  active and noise-tolerant algorithms that remain effective in strategic environments—algorithms that classify strategic agents accurately while issuing as few label requests as possible. The central difficulty is to simultaneously account for strategic manipulation and preserve the efficiency gains of active learning.

Our main result is an algorithm for actively learning linear separators in the strategic setting that preserves the exponential improvement in label complexity over passive learning previously obtained only in the non-strategic case. Specifically, for data drawn uniformly from the unit sphere, we show that a modified version of the Active Perceptron algorithm~\citep{Dasgupta2005Analysis,DBLP:conf/nips/YanZ17} achieves excess error $\varepsilon$ using only $\tilde{O}\!\left(d \ln \tfrac{1}{\varepsilon}\right)$ label queries and incurs at most $\tilde{O}\!\left(d \ln \tfrac{1}{\varepsilon}\right)$ additional mistakes relative to the optimal classifier, even in the {\em nonrealizable} case, when a $\tilde{\Omega}(\varepsilon)$ fraction of inputs have inconsistent labels with the optimal classifier. The algorithm is computationally efficient and, under these distributional assumptions, requires substantially fewer label queries than prior work on strategic Perceptron~\citep{DBLP:conf/sigecom/AhmadiBBN21}.

\end{abstract}

\newpage
\section{Introduction}\label{sec:intro}

Active learning is a well-established framework within machine learning that addresses the problem of efficiently training a model when labeled data is costly or limited~\cite{Dasgupta2005Analysis,Balcan2014}. Rather than assuming the learner has access to all labels, the learner selectively queries labels for the most informative examples. Consider a hiring, admissions, or loan-approval setting where evaluating each candidate requires careful human review. Because each label is expensive, the learner must decide which applications to examine in order to make the most progress in improving model accuracy. Active learning methods formalize this process: they use the features of each example to estimate whether obtaining its label will meaningfully reduce uncertainty and improve the model's predictions. The main challenge of active learning is to design selection rules that allow the learner to identify these high-value examples while keeping the total labeling cost low.

Strategic classification, on the other hand, is an emerging area of modern machine learning research that models scenarios where input features are provided by individuals who might manipulate them to receive better outcomes~\citep{DBLP:conf/kdd/BrucknerS11,DBLP:conf/innovations/HardtMPW16}. This phenomenon arises in many decision-making systems where individuals are evaluated based on observable features. Returning to the same examples, suppose a classifier determines whether someone is shortlisted for an interview, admitted to a program, or offered a favorable loan rate. Individuals may modify their observable features—such as highlighting particular achievements, restructuring their resume, or inflating certain metrics—to increase the likelihood of being classified positively. These adjustments may not reflect genuine improvements in qualification or creditworthiness but are instead strategic changes made in response to the classifier. The main challenge in strategic classification is that the learner only observes manipulated features. This issue is especially critical in an \textit{online setting} where data points arrive sequentially and the learner refines the classification rule over time: individuals react to the current classifier, causing the manipulation behavior to transform—even for similar underlying cases—and the learner's updates rely on data already distorted as a result of manipulation in response to past classifiers. Finally, manipulation behavior can be discontinuous: if the cost of manipulation is just below the perceived benefit of being classified as positive, an individual will manipulate; if just above, they will not. As a result, standard online learning algorithms that achieve bounded mistakes in non-strategic settings may cycle indefinitely and make unbounded errors, even when a perfect classifier exists~\cite{DBLP:conf/sigecom/AhmadiBBN21}.

In this work, we initiate the study of active learning for strategic classification. The main motivation for combining these two areas is that many real-world settings—including the aforementioned scenarios such as hiring, admissions, or loan-approval—must contend both with strategically altered data and costly expert labeling procedures.  In contrast to prior work on strategic classification, which has largely focused on fully supervised settings\footnote{{An independent concurrent work is \cite{zhao2025online} that studies strategic classification with partial feedback.}} where human experts label all training examples, our objective is to design learning algorithms that take a more active role than in the classic fully supervised setting, enabling learning with far fewer label requests. Beyond the individual challenges of active learning and strategic classification, we face unique difficulties arising from their combination. Because we operate in a strategic setting in which features may have been manipulated, there is a risk that the learner requests labels for examples that are ultimately uninformative, derailing the active learning process and leading to classifiers that perform poorly under the original data.

We study an online linear classification problem in which the individuals being classified are strategic and the labels are expensive to obtain. Following the convention in strategic classification~\cite{DBLP:conf/sigecom/AhmadiBBN21}, each individual arriving at the classifier wishes to be classified positively and, if necessary, will manipulate their feature vector to achieve this outcome. More formally, an individual's true feature vector is $\vec{z}$, but they may choose to report a manipulated vector $\vec{x}$ if it results in receiving a positive classification with the current classification rule. The manipulation comes at a cost, which reflects how far their reported vector $\vec{x}$ is from their true vector $\vec{z}$. Specifically, we model individuals as utility-maximizing agents, where the utility is defined as the value received from the classification outcome minus the cost of manipulation. The goal of each individual is to optimize $\vec{x}$ to maximize $\text{value}(\vec{x}) - \text{cost}(\vec{z}, \vec{x})]$, where $\text{value}(\vec{x}) = 1$ if the reported vector $\vec{x}$ is classified as positive and $0$ if classified as negative. The cost function $\text{cost}(\vec{z}, \vec{x})$ quantifies the cost of manipulating the features from $\vec{z}$ to $\vec{x}$. In this setting, if an individual can manipulate their features at a cost of at most 1 to change their classification from negative to positive, they will do so in the least costly way; otherwise, they will not manipulate their features. In line with standard active-learning assumptions~\citep{Dasgupta2005Analysis,DBLP:conf/colt/BalcanBZ07,DBLP:conf/icml/BalcanBL06,Balcan2014,hanneke2014active}, we assume the feature vectors $\vec{z}$ are drawn from an underlying distribution. Formally, we assume that the true feature vectors \( \vec{z} \) of individuals are uniformly distributed within a \( d \)-dimensional unit ball centered at the origin. In the {\em realizable} case, there exists a true classifier \( \vec{u} \) such that \( \vec{u} \cdot \vec{z} \geq 0 \) for all positively labeled points and \( \vec{u} \cdot \vec{z} < 0 \) for all negatively labeled points. In the {\em nonrealizable} case, a certain proportion of points may have labels inconsistent with the best homogeneous linear classifier \( \vec{u} \), reflecting noise in the data. The task for the learning algorithm is to overcome the manipulations and accurately classify the true features, $\vec{z}$, while minimizing the number of labeling queries. 

\paragraph{Our Results and Techniques.} 
The main contribution of this paper is to solve the problem of online learning of linear classifiers with limited label queries in the presence of strategic behavior and in the nonrealizable setting, achieving efficient convergence incurring a bounded excess error over the optimal classifier.
Our main technical contribution is to integrate techniques from strategic classification and active learning, demonstrating how active learning principles can address fundamental limitations in classifying strategic entities. In particular, while prior passive methods struggled in this setting, selectively ignoring certain labels and focusing on informative queries leads to stronger guarantees and more robust performance. We also draw on techniques for noise tolerance in active learning to handle noise in the strategic setting. 

We use the non-strategic active learning algorithms of~\cite{DBLP:conf/nips/YanZ17} as our baseline, and generalize the prediction, label-query, and update rules in several aspects. 
A key adaptation in our algorithm is to focus label queries on unmanipulated examples: we request labels only for points classified as negative and lying in the active learning query region, since utility-maximizing agents do not manipulate when labeled negative, these samples are reliable. Furthermore, symmetry in misclassification ensures that ignoring misclassified positives is harmless: for hyperplanes through the origin, each positive has a corresponding negative with opposite coordinates, and both provide identical directional information for Perceptron-style updates. Consequently, concentrating solely on misclassified negatives preserves the theoretical guarantees while ensuring all queried data reflect true features. Another modification compared to prior work in active learning is raising the threshold for positive classification, as in~\cite{DBLP:conf/sigecom/AhmadiBBN21}: by requiring a strictly positive dot product with the weight vector, truly positive points eventually lie on the correct side (or can strategically manipulate to reach it), while negatives face prohibitively high manipulation costs. Although applied successfully in the realizable setting, the positive threshold could not be paired with prior update rules used in the strategic setting to accommodate the nonrealizable case.
Our analysis includes a reduction from the strategic case to the non-strategic case. 
 
Our main results are summarized as follows--\Cref{thm_main_adversarial_noise} provides a formal statement:
\begin{enumerate}
    \item \textbf{Active Strategic Classification:}  
    We generalize active learning guarantees from non-strategic to strategic settings, providing theoretical and algorithmic foundations for robust classification in the presence of manipulation. In the realizable case, our results imply achieving generalization error $\eps$ after requesting \( \tilde{O}\left(d \ln \frac{1}{\eps}\right) \) labels and making  \( \tilde{O}\left(d \ln \frac{1}{\eps}\right) \) mistakes.

    \item \textbf{Noise in Strategic Classification:} In the nonrealizable case, we achieve excess error $\Theta(\eps)$ after requesting only \( \tilde{O}\left(d \ln \frac{1}{\eps}\right) \) labels and making  an additive \( \tilde{O}\left(d \ln \frac{1}{\eps}\right) \) mistakes compared to the best classifier, when the $\tilde{\Omega}(\eps)$ fraction of the inputs are inverted.   
    We resolve an open problem posed by~\cite{DBLP:conf/sigecom/AhmadiBBN21} regarding handling noise in classification, moving beyond perfect separability. Previous techniques were insufficient for addressing this issue. 
    
\end{enumerate}

\subsection{Related Literature}

Active learning has a long-standing history in machine learning~\citep{Balcan2014}. The central idea is that a learner can achieve better generalization with fewer labeled examples by selectively querying the most informative ones. \cite{DBLP:series/synthesis/2012Settles} provides an extensive survey of active learning algorithms and their applications, highlighting the potential efficiency gains of this approach.
From a theoretical perspective, key paradigms and analysis frameworks include disagreement-based active learning, first studied in the presence of noise by~\cite{DBLP:conf/icml/BalcanBL06}, and further developed by many others~\citep{DBLP:conf/nips/DasguptaHM07,DBLP:journals/jmlr/Koltchinskii10,DBLP:conf/nips/BeygelzimerHLZ10,DBLP:conf/icml/Hanneke07}.
Another widely studied and more practical paradigm is margin-based active learning, where the algorithm queries only examples near the current decision boundary. Our work falls into this category~\citep{DBLP:conf/nips/YanZ17,Dasgupta2005Analysis,DBLP:conf/colt/BalcanBZ07,DBLP:conf/colt/BalcanL13,DBLP:conf/stoc/AwasthiBL14}, and is most closely related to~\cite{DBLP:conf/nips/YanZ17}, as discussed in~\Cref{sec:YanZhang}.

The study of strategic classification has gained increasing attention in recent years, motivated by the need to understand how individuals or entities may ``game'' machine learning systems. The main application domains are admissions, hiring, and financial decision-making (e.g., loan lending), where individuals have incentives to modify their input data.  \cite{DBLP:conf/innovations/HardtMPW16} introduced foundational models of strategic classification, in which individuals manipulate their feature vectors to obtain more favorable outcomes. This line of work has since been extended to examine how classifiers can be designed to be robust to such manipulations~\cite{DBLP:journals/jmlr/BrucknerKS12,DBLP:conf/fat/MilliMDH19}. Several other works explore variations of the strategic classification model, including~\cite{DBLP:conf/sigecom/DongRSWW18,DBLP:conf/forc/BravermanG20,DBLP:conf/nips/HarrisPW23}, but all rely on fully labeled training data. In the application domains of strategic classification, acquiring labeled data—required in all prior work—is often costly, as it typically involves human expert judgment. This highlights the importance of developing algorithms that can learn effectively with fewer labeled examples.

Prior work on online binary ($\pm1$) classification in strategic settings—aimed at optimizing classification accuracy—has primarily focused on two cases: (i) when the original data is perfectly separable, as in~\cite{DBLP:conf/sigecom/AhmadiBBN21}, who showed that guarantees achievable in this setting can fail under even slight inseparability (noise); and (ii) when the data is not separable, but the algorithm’s performance degrades arbitrarily with the level of noise, as in~\cite{DBLP:conf/nips/Chen0P20}. In both cases, prior work falls short of providing robust guarantees in the presence of moderate noise, especially while maintaining low label complexity.

The literatures on strategic classification and active learning both require certain modeling assumptions to obtain theoretical guarantees. Our model for strategic agents follows the framework of~\cite{DBLP:conf/sigecom/AhmadiBBN21}, but we study it in a more challenging active learning setting where the learner selectively queries labels to minimize the need for human intervention. Even in the simpler non-strategic case, active learning requires distributional assumptions to provably achieve label-complexity improvements over passive learning~\citep{Dasgupta2005Analysis,DBLP:conf/colt/BalcanBZ07,DBLP:conf/icml/BalcanBL06,Balcan2014,hanneke2014active}. For clarity of exposition, we adopt one of the most widely used such assumptions for learning linear separators: the true feature vectors are drawn uniformly at random. 
In \Cref{sec:discussion}, we further explain why our results
do not rely critically on this assumption.

The feature data distortion caused by strategic manipulation is fundamentally different from the feature distortion considered in other classic learning theory models such as malicious-noise model~\cite{valiant1984theory,DBLP:conf/stoc/KearnsL88,JMLR:v10:klivans09a,DBLP:conf/stoc/AwasthiBL14,DBLP:journals/jacm/AwasthiBL17}. Although in both settings the observed data is a distorted version of the underlying distribution, the nature of this distortion differs significantly. The malicious-noise model relies critically on the assumption that only a small fraction of examples are corrupted arbitrarily, while the remaining examples are drawn cleanly from the true distribution.
In contrast, in strategic settings {\em every} agent is strategic, and their observed features may be manipulated whenever doing so benefits them. Since arbitrary manipulation yields no meaningful theoretical guarantees, the literature on strategic classification typically assumes that manipulation is {\em structured}: agents adjust their features to optimize their utility in response to the current classifier. Under this assumption, the resulting distortion is distribution-dependent and tightly coupled with the learner’s current classification rule. Algorithms designed for malicious noise, however, rely heavily on having a large fraction of uncorrupted examples in each round—a property that strategic settings do not guarantee.
Consequently, because the structure of distortion differ so substantially between malicious noise and strategic classification, it is unclear whether results from one model should transfer to the other. One of our contributions is to show that a style of {\em localization algorithms}~\cite{DBLP:conf/stoc/AwasthiBL14,DBLP:journals/jacm/AwasthiBL17} developed for malicious-noise tolerance can, in fact, be adapted to strategic classification. 

{The intersection of active learning and strategic classification is a relatively new area of research. While most active learning models assume access to truthful data, strategic classification explicitly accounts for agents who may manipulate their features.
From a theoretical perspective, a key motivation for studying this combination is that ideas from active learning have historically been transformative in improving passive learning algorithms. In several challenging learning problems where existing passive methods failed to obtain optimal guarantees, incorporating active learning techniques has led to major breakthroughs~\citep{DBLP:conf/stoc/AwasthiBL14,DBLP:conf/colt/AwasthiBHU15,DBLP:conf/colt/AwasthiBHZ16,DBLP:journals/jacm/AwasthiBL17}.
In this work, we show that active learning principles can likewise overcome fundamental limitations in classifying strategic entities. In particular, while prior passive algorithms struggled in this setting, selectively querying labels and focusing on informative examples yields stronger theoretical guarantees and more robust performance.
The central challenge we address is how to incorporate active learning techniques when interacting with strategic agents who may manipulate their data, thereby balancing the need for efficient learning with robustness to strategic behavior.}

{\paragraph{Organization of the Paper.} \Cref{sec:model} presents the problem setting, assumptions, and key definitions.
\Cref{sec:noise} introduces our strategic active learning algorithm and proves the main results.
\Cref{sec:discussion} discusses broader implications of our modeling choices and outlines directions for extension.
The appendices contain additional proofs and the initialization procedure of the algorithm.
}

\section{Model and Preliminaries}\label{sec:model}

\paragraph{Strategic Manipulation and Utility Model.}
We study an online classification problem where a sequence of examples in \( \mathbb{R}^d \) arrives one at a time. Each example corresponds to an individual with \( d \) attributes, who wishes to be classified positively. Individuals have the ability to manipulate their attributes at some cost. Let \( \vec{z}_t \) denote the true, unmanipulated instance vector of the \( t \)-th individual, and let \( \vec{x}_t \) be the reported (potentially manipulated) vector observed by the classifier.

 We consider two settings:

\begin{itemize}[ topsep=0pt]
    \item Realizable Case: There exists a true classifier \( \vec{u} \) such that all positive examples satisfy \( \vec{u} \cdot \vec{z}_t \geq 0 \), and all negative examples satisfy \( \vec{u} \cdot \vec{z}_t < 0 \).
    \item Non-Realizable Case: A constant fraction $\nu$ of examples may have labels inconsistent with the classifier \( \vec{u} \), introducing label noise.
\end{itemize}

Individuals are utility-maximizing agents who manipulate their attributes to achieve a positive classification while minimizing manipulation cost. Each individual derives a value of 1 if classified as positive and 0 otherwise. The cost of manipulation, denoted as \( \text{cost}(\vec{z}_t, \vec{x}_t) \), quantifies the effort required to modify \( \vec{z}_t \) to \( \vec{x}_t \). The individual's goal is to maximize:

\[
\max_{\vec{x}_t} [\text{value}(\vec{x}_t) - \text{cost}(\vec{z}_t, \vec{x}_t)].
\]

If positive classification is achievable by spending cost less than 1,
the agent {\em moves}, i.e., changes its feature vector to the {\em least costly} point that ensures a positive classification. Otherwise, they \textit{remain} at \( \vec{z}_t \), i.e., do not manipulate.

We consider the following setting for the Euclidean cost function, where the cost is proportional to the \( \ell_2 \) distance between \( \vec{z}_t \) and \( \vec{x}_t \), i.e.,  
    \(
    \text{cost}(\vec{z}_t, \vec{x}_t) = c \|\vec{x}_t - \vec{z}_t\|_2.
    \)
    Here, \( c \) represents the per-unit movement cost.

\paragraph{Instance Space and Distributional Assumptions.} {We denote the instance space by $\cal Z$ and the label space by $\cal Y$.
The true feature vectors $\vec{z}_t$ belong to instance space ${\cal Z} = \{\vec{z} \in R^d: \|\vec{z}\| \leq 1\}$; a unit $d$-dimensional ball. The label space ${\cal Y} = \{+1, -1\}$.  We assume all examples $\vec{z}$ are drawn i.i.d.\ from the uniform distribution $D$ over $\cal Z$. Upon sampling an example, our algorithm observes $\vec{x}$, whose true instance vector $\vec{z} \in \cal Z$ is drawn from $D$ and whose label is hidden by default. Our algorithm is allowed to make queries to a labeling oracle $\cal O$, which returns the true label for $\vec{z}$.} In line with prior work \citep{Dasgupta2005Analysis,DBLP:conf/nips/YanZ17} on nonstrategic settings,
the goal of the learning algorithm is to classify the true instance vectors accurately while minimizing the number of label queries. To achieve this, we leverage active learning techniques, which allow querying labels only when necessary, reducing reliance on labeled data from a costly oracle.

{In the nonrealizable setting, where there may not be a homogeneous halfspace including all $+1$ and excluding all $-1$ examples, we consider a bounded inseparability (noise) measure $\nu$. Specifically, we say that our setting satisfies the $\nu$-bounded inseparability (noise) condition for some $\nu \in [0, 1]$ with respect to $\vec{u}$, if $\P[Y \neq \text{sign}(\vec{u} \cdot Z)] \leq \nu$.} 

The output of our algorithm is a unit-norm vector $\vec{w}$ defining a halfspace $h$ of the form $\vec{w} \cdot \vec{x} \geq b$, where $b \geq 0$. That is, the resulting halfspace is not necessarily homogeneous. We define the error rate of the halfspace $h$ as $ \err(h) = \P[ \mathbbm{1}(\vec{w}\cdot X \geq b) \neq  Y]$.

For any two vectors $\vec{w_1}, \vec{w_2}$, let $\theta(\vec{w_1}, \vec{w_2}) = \arccos(\vec{w_1} \cdot \vec{w_2})$ be the angle between them. 
Our algorithm uses norm-1 scaled version of the observed examples for the update function. We use the following definition to denote the norm-1 scaled examples.
\begin{definition}[$\hat{\vec{x}}$]
    For any non-zero $d$-dimensional vector $\vec{x}$, we define $\hat{\vec{x}}$ as its scaled version whose length is equal to $1$; i.e., $\hat{\vec{x}} = \frac{\vec{x}}{\|\vec{x}\|}.$
\end{definition}

\paragraph{Learning Objective.}
{Our goal is to design an efficient algorithm such that with probability at least $1-\delta$, outputs a halfspace whose error is at most $\eps$ larger than $\vec{u}$ for $\cal Z$. We require the algorithm to be efficient, use the minimal number of label queries, and make at most $\Theta(\eps)$ mistakes.}

\paragraph{Online Learning Setting.} We assume that the examples arrive online, and each observed example is an input provided by a strategic agent. The agent knows the current classification rule. The algorithm has access to the manipulation cost. Given the current classification rule, the agent selects a utility-maximizing action $\vec{x}$. The algorithm
observes the potentially manipulated example $\vec{x}$. Upon a label query, the algorithm receives the true label of the example. The algorithm operates by a classification rule, a query mechanism, and an update rule at each point in time.

\section{Active and Noise-Tolerant Strategic Perceptron}\label{sec:noise}

In this section, we overcome the challenges of designing active learning algorithms in strategic settings. We propose a modified active Perceptron algorithm that adapts to strategic behavior by selectively querying labels and leveraging the unmanipulated nature of certain points. The modifications ensure that the algorithm remains robust to strategic actions while maintaining the efficiency of active learning. 

{Our main result is as follows.}

\begin{restatable}{theorem}{thmMainAdversarialNoise}
\label{thm_main_adversarial_noise}
Suppose \Cref{alg:active_perceptron} has inputs satisfying the $\nu$-bounded inseparability condition with respect to halfspace \( \vec{u} \), initial halfspace \( \vec{v}_0 \) such that \( \theta(\vec{v}_0, \vec{u}) \leq \pi/2 \), target error \( \eps \), confidence \( \delta \), sample schedule \( \{m_k\} \) where \( m_k = \Theta\left(d \left(\ln d + \ln \frac{k}{\delta}\right)\right) \), and band width \( \{b_k\} \) where \( b_k = \Theta\left(\frac{2^{-k}}{\sqrt{d} \ln(km_k/\delta)}\right) \). Additionally, \( \nu \leq \Theta\left(\frac{\eps}{\ln d + \ln \ln \frac{1}{\eps} + \ln \frac{1}{\delta}}\right) \). Then with probability at least \( 1 - \delta \):
\begin{enumerate}[noitemsep, topsep=0pt]
    \item The output halfspace $\vec{v}$ outputs a prediction different from $\vec{u}$ with probability at most $\eps$.
    \item The number of label queries is \( O\left(d \ln \frac{1}{\eps}  \left(\ln d + \ln \frac{1}{\delta} + \ln \ln \frac{1}{\eps}\right)\right) \).
    \item The number of unlabeled examples drawn is \( O\left(d  \left(\ln d + \ln \frac{1}{\delta} + \ln \ln \frac{1}{\eps}\right)^2  \frac{1}{\eps} \ln \frac{1}{\eps}\right) \).
    \item The additional number of mistakes 
    compared to $\vec{u}$ is \( O\left( d  \ln \frac{1}{\eps}  \left(\ln d + \ln \frac{1}{\delta} + \ln \ln \frac{1}{\eps}\right)^2  \right) \).
    \item The algorithm runs in time \( O\left(d^2  \left(\ln d + \ln \frac{1}{\delta} + \ln \ln \frac{1}{\eps}\right)^2  \frac{1}{\eps} \ln \frac{1}{\eps}\right) \).
\end{enumerate}
\end{restatable}

{We use an initialization procedure, discussed in the appendix, to find \( \vec{v}_0 \) that has an acute angle with the optimal separator.}

\subsection{The Non-Strategic Active Perceptron of~\cite{DBLP:conf/nips/YanZ17}}\label{sec:YanZhang}

Our starting point is the algorithm proposed by Yan and Zhang~\cite{DBLP:conf/nips/YanZ17}. We begin by outlining its main ideas before introducing our modifications for the strategic setting. The algorithm is organized into an outer and an inner layer, similar in structure to~\Cref{alg:active_perceptron} and~\Cref{alg:modified_perceptron}, and proceeds in a sequence of epochs.
The outer layer initializes with a hypothesis vector $\vec{v}_0$ and invokes the inner layer in successive epochs, each with updated parameters such as the target error, confidence level, and active learning bandwidth. The outcome of each epoch is an updated hypothesis $\vec{v}_i$, which serves as the starting point for the next. The total number of epochs is logarithmic in $\nicefrac{1}{\eps}$, where $\eps$ is the final target error.

The inner layer defines both the update rule and the label-query mechanism. In the non-strategic setting, it specifies a label-query region $R_t$; in the implementation of~\cite{DBLP:conf/nips/YanZ17}, which assumes that examples lie on the unit sphere, this region is the set of points whose dot product with the current hypothesis lies in \( [\nicefrac{b}{2}, b] \). Compared to earlier active Perceptron algorithms such as~\cite{Dasgupta2005Analysis}, the use of a lower bound on the dot product ensures that each update makes sufficient progress, thereby accelerating convergence. As for the update rule, when the algorithm makes a mistake on a queried example, it updates the current hypothesis using $\vec{w}_{t+1} = \vec{w}_t \pm 2(\vec{w}_t \cdot \vec{x}_t)\vec{x}_t$. This update rule, first proposed by~\cite{Dasgupta2005Analysis}, guarantees that the angle between the current hypothesis and the optimal separator $\vec{u}$ monotonically decreases. Furthermore, it preserves the unit norm of the hypothesis vector as long as $|\vec{x}_t| = 1$.

The convergence builds on the fact that, with high probability, both the angle between the current hypothesis and $\vec{u}$ and the width of the label-query region shrink by a constant factor at each epoch.

\subsection{Our Strategic Variant of the Active Perceptron}

Similar to~\cite{DBLP:conf/nips/YanZ17}, we consider an outer layer,~\Cref{alg:active_perceptron}, and inner layer,~\Cref{alg:modified_perceptron}, for our algorithm, where the new prediction rule, label query region, and update rule are incorporated in the inner layer. 

We begin by explaining the new prediction rule that adjusts the classification threshold to account for manipulation costs. We then show that, under our strategic utility model, any example that is predicted negatively has not been manipulated.

\paragraph{Prediction Rule.}
Rather than using the standard threshold $\vec{v}_t \cdot \vec{x}_t \geq 0$, our prediction rule raises the threshold to $\vec{v}_t \cdot \vec{x}_t \geq \frac{1}{c}$, where $c$ is the cost per unit of manipulation. This adjustment, following~\cite{DBLP:conf/sigecom/AhmadiBBN21}, accounts for agents' strategic behavior and ensures that, conditioned on convergence to the optimal classifier, (the majority of the) truly positive points either lie on the positive side or can manipulate to reach it. Meanwhile, truly negative points remain on the negative side and would incur negative utility if they attempted to manipulate and be classified as positive.

To analyze agent behavior in the strategic setting, we begin by characterizing their actions under the given utility structure and prediction rule. The following result formalizes the conditions under which agents choose to manipulate their features and the resulting outcomes. In particular, it shows that examples classified as negative are guaranteed to be unmanipulated, a property that is essential for ensuring the correctness of our update rule. For completeness, we provide the proof of correctness in the appendix.

\begin{restatable}[Strategic Action]{lemma}{lmStrategicAction}
\label{lm:strategic_action}
Consider the following utility structure for agents, where $\|\vec{v}_t\| = 1$. Each agent receives a value of 1 if classified as positive and $0$ otherwise, and pays a cost of $c$ per unit of movement (manipulation). The agent's utility is defined as the value received minus the cost incurred. Under prediction rule $\mathbbm{1}[\vec{v}_t \cdot \vec{x} \geq 0]$:
\begin{enumerate}[noitemsep, topsep=0pt]
\item If $\vec{z}_t \cdot \vec{v}_t < 0$, the agent does not move; i.e., $\vec{x}_t = \vec{z}_t$, and is classified negative.
\item If $0 \leq \vec{z}_t \cdot \vec{v}_t < 1/c$, the agent moves in the direction of $\vec{v}_t$ to a point where $\vec{x}_t \cdot \vec{v}_t = 1/c$, i.e., $\vec{x}_t = \vec{z}_t + (\nicefrac{1}{c}-\vec{z}_t \cdot \vec{v}_t)\vec{v}_t$, and is classified positive.
\item If $1/c \leq \vec{z}_t \cdot \vec{v}_t$, the agent does not move, i.e., $\vec{x}_t = \vec{z}_t$, and is classified positive.
\end{enumerate}
\end{restatable}
\paragraph{Label Query Region.} In the modified version of the algorithm, we query labels (and perform updates) only for examples that are classified as negative and lie within a specific range. We define this label-requesting region as
$
R_t = \left\{ \vec{x} \,\middle|\, -b \leq \vec{w}_t \cdot \hat{\vec{x}} \leq \frac{-b}{2} \right\}.
$
This design is essential for both addressing the strategic behavior of agents and accommodating instance vectors that are not restricted to the surface of the unit sphere. It marks a key point of departure from both classical active Perceptron algorithms and previous work on strategic Perceptron.
(1) Since we focus on negatively classified examples, \Cref{lm:strategic_action} guarantees that these examples are unmanipulated; that is, the observed vector $\vec{x}$ coincides with the true vector $\vec{z}$. This property does not hold for positively classified examples, and thus plays no role in non-strategic active learning.
(2) Unlike prior work that restricts attention to examples on the surface of the unit sphere, our algorithm also queries examples from the interior. For such queries to be representative under a uniform distribution over the unit ball, it is critical that the observed (i.e., unmanipulated) examples remain uniformly distributed—a property that does not hold if examples are manipulated.
(3) As we show in \Cref{lm:strategic_action}, querying within $R_t$ yields uniformly distributed samples (after normalization) conditioned on being in that region and classified negative. This ensures the correctness of the updates and maintains the convergence behavior of the algorithm.

\paragraph{Update Rule.}
The update rule requires only minimal modifications to account for strategic behavior and a more general instance space. Since examples may lie anywhere within the unit ball, we normalize each queried example to the unit sphere by setting $\hat{\vec{x}}_t = \vec{x}_t / \|\vec{x}_t\|$. This normalization ensures compatibility with the geometric assumptions underlying the analysis and avoids distortions due to varying magnitudes.
Updates are performed only on examples that are truly positive but misclassified as negative. As established in \Cref{lm:strategic_action}, such examples are guaranteed to be unmanipulated and therefore reflect the true feature vectors of the agents.
The update step itself takes the form
$\vec{v}_{t+1} = \vec{v}_t + 2\left(\vec{v}_t \cdot \hat{\vec{x}}_t \right)\hat{\vec{x}}_t$,
which mirrors the standard active Perceptron update, except for the added normalization. This scaling step is essential for maintaining the unit norm of the hypothesis vector, which in turn ensures the correct convergence behavior of the algorithm.

\begin{algorithm}[H]
\caption{Active-Strategic-Perceptron Algorithm}\label{alg:active_perceptron}
\begin{algorithmic}
\State \textbf{Input:} Labeling oracle $\cal O$, initial halfspace $\vec{v}_0$, target error $\eps$, confidence $\delta$, sample schedule $\{m_k\}$, band width $\{b_k\}$, manipulation cost $c$.
\State \textbf{Output:} Learned halfspace $\vec{v}$.
\State Let $k_0 = \lceil \log_2 (1/\eps) \rceil$.\\
\For{$k = 1, 2, \ldots, k_0$}
    {{\State $\vec{v}_k \gets \text{Modified-Strategic-Perceptron}({\cal O}, \vec{v}_{k-1}, \frac{\pi}{2^k}, \frac{\delta}{k(k+1)}, m_k, b_k, c)$.}}
\State \Return $\vec{v}_{k_0}$.
\end{algorithmic}
\end{algorithm}

\begin{algorithm}[H]
\caption{Modified-Strategic-Perceptron Algorithm}\label{alg:modified_perceptron}
\begin{algorithmic}
\State \textbf{Input:} Labeling oracle $\cal O$, initial halfspace $\vec{w}_0$, angle upper bound $\theta$, confidence $\delta$, number of iterations $m$, band width $b$, manipulation cost $c$.
\State \textbf{Output:} Improved halfspace $\vec{w}_m$.\\
\For{$t = 0, 1, 2, \ldots, m-1$}
    {\State Define region {$R_t = \{{\vec{x}} \mid {-b \leq \vec{w}_t \cdot \hat{\vec{x}} \leq \frac{-b}{2}}\}$}.\\
    Observe $\vec{x}$, {where its corresponding true feature vector $\vec{z}$ is a fresh draw from $D$}.\\
    \While{$\hat{\vec{x}} \notin R_t$}
        {\State Predict positive if  $\vec{{v}}_t \cdot \vec{x} \geq \frac{1}{c}$, and negative otherwise.
        \State Observe $\vec{x}$, {where its corresponding true feature vector $\vec{z}$ is a fresh draw from $D$}.}
    \State $\vec{x}_t \gets \vec{x}$.
    \State Predict positive if  $\vec{{v}}_t \cdot \vec{x}_t \geq \frac{1}{c}$, and negative otherwise.
    \State Observe label $y_t$ of $\vec{x}_t$ by querying oracle $\cal O$.\\
    \If{$y_t = +1$}
        {\State Update $\vec{w}_{t+1} \gets \vec{w}_t + 2 (\vec{w}_t \cdot {\hat{\vec{x}}_t}){\hat{\vec{x}}_t}$.} 
    \State \Return $\vec{w}_m$.}
\end{algorithmic}
\end{algorithm}

\subsection{Proof of the Main Result}

We provide the proof of our main result, \Cref{thm_main_adversarial_noise}, in this section.

First, we establish that the prediction rule identified by $\vec{v}_i$ in the strategic setting using adjusted threshold $1/c$, predicts exactly the same as $\vec{v}_i$ in the nonstrategic setting.
The following is a direct corollary of~\Cref{lm:strategic_action}.

\begin{corollary}\label{cor:same_prediction}
    The prediction rule $\mathbbm{1}[\vec{x}\cdot\vec{v} \geq 1/c]$ classifies all $\vec{z}$ such that $\vec{z}\cdot\vec{v} \geq 0$ as $+1$ and all $\vec{z}$ such that $\vec{z}\cdot\vec{v} < 0$ as $-1$, where $\vec{z}$ is the true feature vector corresponding to observed feature vector $\vec{x}$.
\end{corollary}

Next, we invoke a result from~\cite{DBLP:conf/nips/YanZ17} that shows in the nonstrategic setting, the difference in the error rate of two prediction rules can be related to the angles of their normal vectors.

\begin{lemma}[\cite{DBLP:conf/nips/YanZ17} Lemma 1]\label{lm:error_diff_nonstrategic}
    In the nonstrategic setting, for any unit-sized vector $\vec{v}_1$ and $\vec{v}_2$ and prediction rules $\mathbbm{1}[\vec{z} \cdot \vec{v}_i \geq 0]$, 
    \[|\err(\vec{v}_1)-\err(\vec{v}_2)| \leq \P [\mathbbm{1}[\vec{z} \cdot \vec{v}_1 \geq 0] \neq \mathbbm{1}[\vec{z} \cdot \vec{v}_2 \geq 0]] = \frac{\theta(\vec{v}_1,\vec{v}_2)}{\pi}.\]
\end{lemma}

Now, we extend the above argument to the strategic setting.

\begin{lemma}\label{lm:same_set}
    Consider the prediction rule of \Cref{alg:modified_perceptron} for for unit-sized vectors $\vec{v}_1$ and $\vec{v}_2$. 
    \[|\err(h_{\vec{v}_1}) - \err(h_{\vec{v}_2})| \leq \frac{\theta(\vec{v}_1, \vec{v}_2)}{\pi},\] where halfspace $h_{\vec{v}_i}: \vec{v}_i \cdot x \geq 1/c$ demonstrates the positive prediction area for observed feature vectors $\vec{x}$, with respect to $\vec{v}_i$. 
    \end{lemma}
\begin{proof}
    The prediction rule for $\vec{v}_i$ is 1 iff $\vec{v}_i \cdot \vec{x} - 1/c \geq 0$. Consider an arbitrary example $\vec{z}$ with its utility maximizing action $\vec{x}$ with respect to prediction rule $\vec{v}_i$. By \Cref{cor:same_prediction}, the set of $\vec{z}$ classified as positive or negative under the rule $\vec{v}_i \cdot \vec{z} \geq 0$ in the non-strategic setting exactly matches the set classified as positive or negative under $\vec{v}_i \cdot \vec{x} \geq 1/c$ in the strategic setting, respectively. Invoking \Cref{lm:error_diff_nonstrategic} for the nonstrategic setting implies that the difference in error in the strategic setting is also at most $\theta(\vec{v}_1,\vec{v}_2)/\pi$.
\end{proof}

A key remaining ingredient to extend the guarantees from the nonstrategic case is to show that the \textit{update rule} in our setting, i.e., strategic agents and arbitrary-sized vectors $\vec{z}$ drawn from the uniform distribution, behaves similarly to the nonstrategic case and unit-sized inputs. 

As the next step, we show that at time $t$ the observed examples satisfying $\vec{x}_t \cdot \vec{v}_t < 1/c$ are distributed according to the underlying distribution and satisfy bounded separability. This fact is a corollary of~\Cref{lm:strategic_action}.

\begin{corollary}\label{cor:prediction_rule_not_affected}
    For any negatively classified example at time $t$, i.e., $\vec{x}_t \cdot \vec{v}_t < 1/c$, the observed attributes are the same as the true attributes, i.e., $\vec{x}_t = \vec{z}_t$ and $\vec{z}_t \cdot \vec{v}_t < 0$. Furthermore, such examples are distributed uniformly according to $D$ subject to $\vec{z}_t \cdot \vec{v}_t < 0$ and satisfy the $\nu$ bounded separability.
\end{corollary}

The above corollary implies that the strategic aspect does not affect the prediction step. 
Next, we show that using our update rule induces a coupling between examples inside the unit sphere and those on its surface.

\begin{lemma}\label{lm:update_rule_not_affected}
Consider the update rule $\vec{w}_{t+1} \gets \vec{w}_t + 2 (\vec{w}_t \cdot \hat{\vec{x}}_t)\hat{\vec{x}}_t$, where $\vec{x}_t$ is sampled from distribution $D$ conditioned on $-b \leq \vec{w}_t \cdot \hat{\vec{x}}_t \leq \frac{-b}{2}$. The distribution over updated vectors $\vec{w}_{t+1}$ matches that of the update rule $\vec{w}_{t+1} \gets \vec{w}_t + 2 (\vec{w}_t \cdot \vec{x}_t)\vec{x}_t$, where $\vec{x}_t$ is sampled uniformly from the $\|\vec{x}\| = 1$ conditioned on $-b \leq \vec{w}_t \cdot \vec{x}_t \leq \frac{-b}{2}$.
\end{lemma}

\begin{proof}
We construct a coupling between the samples used in the two update rules. Fix a unit vector $\hat{\vec{x}}$ with $\|\hat{\vec{x}}\| = 1$, and consider the set of points $\vec{x}$ such that $\vec{x} / \|\vec{x}\| = \hat{\vec{x}}$. Under the distribution $D$ (uniform over the unit ball), the density over such $\vec{x}$ induces a uniform distribution over directions $\hat{\vec{x}}$ on the unit sphere. 
Since the update rule using $\hat{\vec{x}}$ depends only on the direction of $\vec{x}$ (and not its norm), and since every $\vec{x}$ whose unit-sized scaled version is $\hat{\vec{x}}$ yields the same update vector, the two update distributions are equivalent. That is, sampling from $D$ and normalizing before updating yields the same distribution over $\vec{w}_{t+1}$ as sampling directly from the unit sphere and updating without normalization.
\end{proof}

By \Cref{cor:prediction_rule_not_affected} and \Cref{lm:update_rule_not_affected}, we can disregard the effects of strategic behavior and the arbitrary norms of the examples, and directly apply the guarantees established in the non-strategic setting:

\begin{theorem}[Adapted from \citep{DBLP:conf/nips/YanZ17}, Theorem 3]\label{thm:YZ_adversarial_thm}
    Suppose \Cref{alg:active_perceptron} has inputs satisfying the $\nu$-bounded inseparability condition with respect to halfspace \( \vec{u} \), initial halfspace \( \vec{v}_0 \) such that \( \theta(\vec{v}_0, \vec{u}) \leq \pi/2 \), target error \( \eps \), confidence \( \delta \), sample schedule \( \{m_k\} \) where \( m_k = \Theta\left(d \left(\ln d + \ln \frac{k}{\delta}\right)\right) \), and band width \( \{b_k\} \) where \( b_k = \Theta\left(\frac{2^{-k}}{\sqrt{d} \ln(km_k/\delta)}\right) \). Additionally, \( \nu \leq \Theta\left(\frac{\eps}{\ln d + \ln \ln \frac{1}{\eps} + \ln \frac{1}{\delta}}\right) \). Then with probability at least \( 1 - \delta \):
\begin{enumerate}[noitemsep, topsep=0pt]
    \item The output halfspace $\vec{v}$ outputs a prediction different from $\vec{u}$ with probability at most $\eps$.
    \item The number of label queries is \( O\left(d \ln \frac{1}{\eps} \cdot \left(\ln d + \ln \frac{1}{\delta} + \ln \ln \frac{1}{\eps}\right)\right) \).
    \item The number of unlabeled examples drawn is \( O\left(d \cdot \left(\ln d + \ln \frac{1}{\delta} + \ln \ln \frac{1}{\eps}\right)^2 \cdot \frac{1}{\eps} \ln \frac{1}{\eps}\right) \).
    \item The algorithm runs in time \( O\left(d^2 \cdot \left(\ln d + \ln \frac{1}{\delta} + \ln \ln \frac{1}{\eps}\right)^2 \cdot \frac{1}{\eps} \ln \frac{1}{\eps}\right) \).
\end{enumerate}

\end{theorem}

\begin{lemma}[Adapted from \citep{DBLP:conf/nips/YanZ17}, Lemma 3]\label{lm:YZ_key_lemma}
    Suppose \Cref{alg:modified_perceptron} has inputs satisfying the $\nu$-bounded inseparability condition with respect to halfspace \( \vec{u} \), initial halfspace \( \vec{w}_0 \) and angle upper bound $\theta \in (0, \pi/2]$ such that \( \theta(\vec{w}_0, \vec{u}) \leq \theta \), confidence \( \delta \), 
    number of iterations \( m = \Theta\left(d \left(\ln d + \ln \frac{1}{\delta}\right)\right) \), and band width \( b = \Theta\left(\frac{\theta}{\sqrt{d} \ln(m/\delta)}\right) \). Additionally, $\nu = O\left(\frac{\theta}{\ln{(m/\delta)}}\right)$. 
    Then with probability at least \( 1 - \delta \):
\begin{enumerate}[noitemsep, topsep=0pt]
    \item The output $\vec{w}_m$ is such that $\theta(\vec{w}_m, \vec{u}) \leq \frac{\theta}{2}$.
    \item The number of label queries is $O\left( d \cdot \left( \ln d + \ln \frac{1}{\delta} \right) \right)$.
    \item The number of unlabeled examples drawn is $O\left( d \cdot \left( \ln d + \ln \frac{1}{\delta} \right)^2 \cdot \frac{1}{\theta} \right)$.
    \item The algorithm runs in time $O\left( d^2 \cdot \left( \ln d + \ln \frac{1}{\delta} \right)^2 \cdot \frac{1}{\theta} \right)$.
\end{enumerate}

\end{lemma}

The only remaining part is proving the mistake bound of the algorithm.

\begin{lemma}\label{lm:mistake_bound}
    Suppose \Cref{alg:active_perceptron} has inputs satisfying the $\nu$-bounded inseparability condition with respect to halfspace \( \vec{u} \), initial halfspace \( \vec{v}_0 \) such that \( \theta(\vec{v}_0, \vec{u}) \leq \pi/2 \), target error \( \eps \), confidence \( \delta \), sample schedule \( \{m_k\} \) where \( m_k = \Theta\left(d \left(\ln d + \ln \frac{k}{\delta}\right)\right) \), and band width \( \{b_k\} \) where \( b_k = \Theta\left(\frac{2^{-k}}{\sqrt{d} \ln(km_k/\delta)}\right) \). Additionally, \( \nu = O \left(\frac{\eps}{\ln d + \ln \ln \frac{1}{\eps} + \ln \frac{1}{\delta}}\right) \). Then with probability at least \( 1 - \delta \), The additional number of mistakes that the algorithm makes compared to $\vec{u}$ is  \( O\left( d \cdot \ln \frac{1}{\eps} \cdot \left(\ln d + \ln \frac{1}{\delta} + \ln \ln \frac{1}{\eps}\right)^2  \right) \).

\end{lemma}

\begin{proof}
Similar to the proof of Theorem 3 in \citep{DBLP:conf/nips/YanZ17}, we define, for each iteration $k$, a corresponding event with high individual success probability. Specifically, by \Cref{lm:YZ_key_lemma}, for every $k$, there exists an event $E_k$ such that $\Pr(E_k) \geq 1 - \frac{\delta}{k(k+1)}$. Moreover, on event $E_k$, items 1 through 4 of \Cref{lm:YZ_key_lemma} hold for the input $\vec{w}_0 = \vec{v}_k$.

The excess error of $\vec{w}_t$ relative to $\vec{u}$ is $\theta_t / \pi$, where $\theta_t$ is the angle between $\vec{w}_t$ and $\vec{u}$. Consider \Cref{alg:modified_perceptron} with initial halfspace $\vec{w}_0$ and angle bound $\theta \in [0, \frac{\pi}{2}]$ such that $\theta(\vec{w}_0, \vec{u}) \leq \theta$. Let the number of iterations be $m = \Theta\left(\frac{d}{(1 - 2\eta)^2} \left(\ln \frac{d}{(1 - 2\eta)^2} + \ln \frac{k}{\delta} \right)\right)$. Then, with probability at least $1 - \frac{\delta}{k(k+1)}$, the output halfspace $\vec{w}_m$ satisfies $\theta(\vec{w}_m, \vec{u}) \leq \frac{\theta}{2}$.

By items 2 and 3 in \Cref{lm:YZ_key_lemma}, with probability at least $1-\frac{\delta}{k(k+1)}$, the total number of examples seen during epoch $k$ is: 
$O\left( d \cdot \left(\ln d + \ln \frac{k}{\delta} \right)^2 \cdot \frac{1}{\theta_t} \right)$. Since each example differs in classification from $\vec{u}$ with probability $\frac{\theta}{\pi}$, the number of additional misclassified examples compared to $\vec{u}$ is $O\left( d \cdot \left(\ln d + \ln \frac{k}{\delta} \right)^2 \right)$. The number of total time epochs, $k_0 =  \lceil \log_2(1/\eps) \rceil$ and $k \leq k_0$. 
The total number of epochs is $k_0 = \lceil \log_2(1/\eps) \rceil$, and we have $k \leq k_0$. Therefore, by a union bound over all epochs, with probability at least $1 - \delta$, the total number of additional mistakes compared to $\vec{u}$ is
\( O\left( d \cdot \ln \frac{1}{\eps} \cdot \left(\ln d + \ln \frac{1}{\delta} + \ln \ln \frac{1}{\eps}\right)^2  \right) \).

\end{proof}


\begin{proof}[Proof of~\Cref{thm_main_adversarial_noise}]
\Cref{cor:prediction_rule_not_affected}, \Cref{lm:update_rule_not_affected} and \Cref{thm:YZ_adversarial_thm} imply items 1, 2, 3, and 5. \Cref{lm:mistake_bound} implies item 4.
\end{proof}

\section{Discussion}\label{sec:discussion}

In this section, we provide a discussion of several of our modeling assumptions and extensions of our results beyond them.

\paragraph{Inside vs. On-the-Surface Geometric Assumptions.}  
Much of the prior literature on active classification, e.g.,~\cite{Dasgupta2005Analysis,DBLP:conf/nips/YanZ17}, assumed for simplicity and cleaner mathematical formulations, that all examples lie on the surface of a unit ball; for illustration, the circle in~\Cref{fig:surface_ball}. While this assumption was inconsequential to the goals of previous research, it becomes crucial in strategic scenarios, where the relationship between observed and true features is strongly influenced by the geometry of the space. In the strategic setting, this assumption implies a one-to-one mapping between observed feature vectors, $\vec{x}_i$ in~\Cref{fig:surface_ball} and original feature vectors, $\vec{z}_i$, and provides a straightforward case for analysis, as it allows the original unmanipulated positions of examples to be completely recovered. Specifically, given a linear classifier, e.g., $\vec{v}\cdot \vec{x} \geq 1/c$, and observing an example at position \( \vec{x} \), the true position \( \vec{z} \) of the example on the surface of the unit ball can be recovered through an orthogonal projection, leveraging properties of the utility function, as shown in~\Cref{lm:strategic_action}. %

\begin{figure}[h] 
    \centering
\begin{tikzpicture}[scale=1]  \def\R{2} 

  \draw[thick] (0,0) circle (\R);

  \draw[line width=2pt, purple] (90:\R) arc (90:270:\R);
    \draw[line width=2pt, purple] (-60:\R) arc (-60:60:\R);

  \draw[dashed] (0,-2.5) -- (0,2.6);
  \draw[dashed] (1,-2.5) -- (1,2.6)
    node[above =1pt] {$\vec{v}\cdot \vec{x} = 1/c$};

  \draw[line width=2pt, purple] (1,1.72) -- (1,2.03);

  \draw[line width=2pt, purple] (1,-1.72) -- (1,-2.03);

  \fill (180:\R) circle (1.5pt)
        node[left=4pt] {$\vec{x}_3 = \vec{z}_3$};

  \fill (-30:\R) circle (1.5pt)
        node[right=4pt] {$\vec{x}_2 = \vec{z}_2$};
        
  \fill (72.5:\R) circle (1.5pt)
        node[above =1pt] {$\vec{z}_1$};

  \fill (1,1.9) circle (1.5pt)
        node[above right=1pt] {$\vec{x}_1$};

 \draw[dotted] (72.5:\R) -- (1,1.9);

\end{tikzpicture}
    \caption{Illustration of observed regions of examples, where the original examples are on the surface of a unit ball, as assumed in prior literature on active learning~\cite{Dasgupta2005Analysis,DBLP:conf/nips/YanZ17}. The colored area shows the potential positions where the examples can be observed, following~\Cref{lm:strategic_action}. Under this assumption, there is a one-to-one mapping between the original examples and observed ones, and observing an example $\vec{x}_i$ exactly identifies the original position $\vec{z}_i$. For instance, when observing $\vec{x}_1$ to recover the original point $\vec{z}_1$, one needs to consider the projection of $\vec{x}_1$ on the surface of the pall, in the opposite direction of $\vec{v}$. This one-to-one mapping no longer exists under the relaxed assumptions that original examples can be inside the ball and not just the surface.}
    \label{fig:surface_ball}
\end{figure}
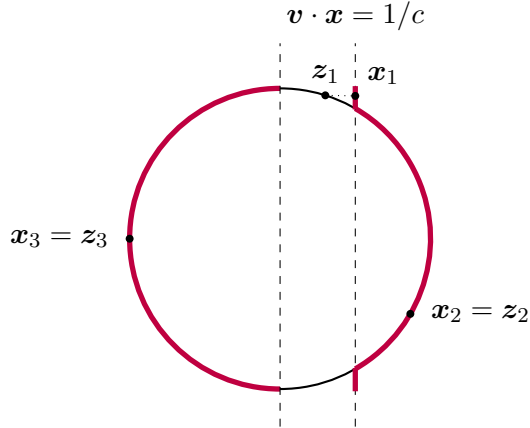

To illustrate the robustness of our techniques, we relax this assumption—an aspect that was less relevant in prior work due to their different focus. While the direction of manipulation can still be computed based on the linear classifier in action (as the direction is always perpendicular to the classifier), the true unmanipulated position \( \vec{z} \) is not recoverable from the observed \( \vec{x} \) because the magnitude of the manipulation is unknown.

\paragraph{Beyond Uniform Distribution.} Although we focus on the uniform distribution for simplicity, our results do not rely strictly on uniformity and the ideas we present extend more generally. In particular, our results extend—up to constant-factor degradations in the guarantees—to any underlying distribution whose probability densities of true examples belonging to each ray from the center of the ball differ by at most a constant factor.

\paragraph{Other Active Learning Algorithms as Baseline.} Although we have focused on the active learning algorithm by~\cite{DBLP:conf/nips/YanZ17} as our baseline, and presented a modified version of this algorithm for the strategic setting, our ideas have the potential to the be applied to other such algorithms. Particularly, following the new prediction rule and label query region, tailored to strategic setting, our analysis provides a reduction from the strategic scenario to non-strategic ones that could be applicable to a variety of active classification algorithms.

\paragraph{Unknown Costs and other Cost Functions.} 
Finally, we conjecture using similar ideas as~\cite{DBLP:conf/sigecom/AhmadiBBN21}, our algorithms can be generalized to scenarios where the learner does not know the manipulation costs of strategic agents upfront and needs to learn throughout the algorithm, and scenarios with different utility and cost functions, e.g., weighted $\ell_1$ costs of manipulation.

\bibliographystyle{alpha}
\bibliography{main}
\newpage
\appendix

\section{Missing Proofs Related to Strategic Behavior}

The following statement characterizes action $\vec{x}$ as a function of $\vec{z}$ and the prediction rule, where $\vec{x}$ is the observed attributes vector by the algorithm.

\lmStrategicAction*
\begin{proof}
    Using the prediction rule in \Cref{alg:modified_perceptron}, an agent with observed vector $\vec{x}$ is classified as positive iff $\vec{v}_t \cdot \vec{x} \geq 1/c$. We analyze the utility-maximizing behavior of agents under the given utility structure in \Cref{sec:model}. The utility of an agent is defined as:
    \[
    \text{Utility}(\vec{x}) = \begin{cases} 
    1 - c \|\vec{x} - \vec{z}_t\| &   \text{if } \vec{v}_t \cdot \vec{x} \geq \frac{1}{c}, \\
    - c \|\vec{x} - \vec{z}_t\| & \text{otherwise.}
    \end{cases}
    \]
    As implied from Observation 1 in \cite{DBLP:conf/sigecom/AhmadiBBN21}, the agent with attributes vector $\vec{z}_t$ only manipulates iff by movement cost at most $c$ can be classified as positive.  
    We now analyze the three cases:
    (i) If $\vec{z}_t \cdot \vec{v}_t < 0$, any movement incurs negative utility, so the agent does not move.
    (ii) If $0 \leq \vec{z}_t \cdot \vec{v}_t < 1/c$, the agent maximizes utility by moving to $\vec{x} \cdot \vec{v}_t = 1/c$ in direction $\vec{v}_t$, i.e., \[\vec{x} = \vec{z}_t + (1/c - \vec{z}_t \cdot \vec{v}_t) \cdot \vec{v}_t.\]
    This is the minimum manipulation cost to be classified as positive.
    (iii) Finally, if $\vec{z}_t \cdot \vec{v}_t \geq 1/c$, the agent already achieves maximum utility without moving.
\end{proof}

\section{Initialization Step}

\Cref{alg:active_perceptron} assumes that the initial vector $\vec{v}_0$ forms an angle less than $\pi/2$ with the true separator $\vec{u}$. To obtain such a vector, we adopt an approach similar to that of \cite{DBLP:conf/nips/YanZ17,DBLP:conf/stoc/AwasthiBL14}. Intuitively, \Cref{alg:master_adversarial_noise} begins with two opposite vectors and, after a small number of trials, selects the one with lower classification error. This initialization step incurs only a constant overhead in terms of label, time, and mistake complexities.

\begin{algorithm}[H]
\caption{Master Algorithm in Adversarial Noise  Setting, (adapted from \cite{DBLP:conf/nips/YanZ17,DBLP:conf/stoc/AwasthiBL14})}\label{alg:master_adversarial_noise}
\begin{algorithmic}
\State \textbf{Input:} Labeling oracle $\cal O$, confidence $\delta$.
\State \textbf{Output:} A halfspace $\vec{v}$ such that $\theta(\hat{\vec{v}}, \vec{u}) \leq \pi/4$.
\State $\vec{v}_0 \gets (1,0,\ldots,0)$.
\State $\vec{v}_+ \gets \text{Active-Strategic-Perceptron}({\mathcal O}, \vec{v}_0, \frac{1}{16}, \frac{\delta}{3}, \{m_k\}, \{b_k\}, \{c\})$.
\State $\vec{v}_- \gets \text{Active-Strategic-Perceptron}({\mathcal O}, -\vec{v}_0, \frac{1}{16}, \frac{\delta}{3}, \{m_k\}, \{b_k\}, \{c\})$.
\State {If $\theta(\vec{v}_{-},\vec{v}_{+})\leq \pi/20$ return $\vec{v}_{+}$.}
\State {Define region $R := \{\vec{Z}: \sign(\vec{v}_+ \cdot \vec{z} ) \neq \sign(\vec{v}_- \cdot \vec{z}) \}$.}
\State {$S = \emptyset$};
\State {$m_{-}= m_{+} = 0$}

\\
\While{$|S| < 8\ln\frac{6}{\delta}$}
    {\State Observe $\vec{x}$, {where $\vec{z}$ is a fresh draw from $D$}.
    \State Predict negative.\\
    \If {$\vec{x}$ is in $R$} 
    {\State Query label $Y$.
    \State $m_{-} += \mathbf{1}\{\sign{\vec{v}_{-}}\cdot \vec{x} = Y\}$
    \State $m_{+} += \mathbf{1}\{\sign{\vec{v}_{-}}\cdot \vec{x} = Y\}$
    \State $S = S \cup \{\vec{x}\}$
    }}\\
\If{$m_+ \leq m_-$}
    {{\State \Return $\vec{v}_+$.}}
\Else
    {\State \Return $\vec{v}_-$.}
\end{algorithmic}
\end{algorithm}

By \Cref{cor:prediction_rule_not_affected} and \Cref{thm_main_adversarial_noise}, we can effectively disregard the strategic behavior of agents and apply Theorem 12 from \cite{DBLP:conf/nips/YanZ17}, which implies the following.

\begin{theorem}
Suppose \Cref{alg:master_adversarial_noise} has inputs labeling oracle $\mathcal{O}$ that satisfies $\nu$-inseparability condition with respect to $\vec{u}$, confidence $\delta$, and sample schedule $\{m_k\}$ where
\[
m_k = \Theta\left( d \left( \ln d + \ln \frac{k}{\delta} \right) \right),
\]
and band width $\{b_k\}$ where
\[
b_k = \tilde{\Theta}\left( \frac{2^{-k}}{\sqrt{d}} \right).
\]
Then, with probability at least $1 - \delta$, the output $\hat{\vec{v}}$ is such that $\theta(\hat{\vec{v}}, \vec{u}) \leq \frac{\pi}{4}$. Furthermore:
\begin{enumerate}
    \item the total number of label queries to oracle $\mathcal{O}$ is at most $\tilde{O}(d)$;
    \item the total number of unlabeled examples drawn is $\tilde{O}(d)$;
    \item the total number of additional mistakes compared to $\vec{u}$ is $\tilde{O}(d)$;
    \item the algorithm runs in time $\tilde{O}(d^2)$.
\end{enumerate}
\end{theorem}

\end{document}